\documentclass[sigconf, nonacm]{acmart}
\AtBeginDocument{%
  }

\input{generic_aliases}

\copyrightyear{2025}
\acmYear{2025}
\setcopyright{cc}
\setcctype{by}

\begin{document}

\title{Bypassing Skip-Gram Negative Sampling:\\Dimension Regularization as a More Efficient Alternative for Graph Embeddings}

\author{David Liu}
\affiliation{%
  \institution{Northeastern University}
  \department{Khoury College of Computer Sciences}
  \city{Boston}
  \state{MA}
  \country{USA}
}

\author{Arjun Seshadri}
\affiliation{%
  \institution{Amazon}
  \department{AWS}
  \city{San Francisco}
  \state{CA}
  \country{USA}
}

\author{Tina Eliassi-Rad}
\affiliation{%
  \institution{Northeastern University}
  \department{Khoury College of Computer Sciences \& \\Network Science Institute}
  \city{Boston}
  \state{MA}
  \country{USA}
}

\author{Johan Ugander}
\affiliation{%
  \institution{Stanford University}
  \department{Management Science \&\\Engineering}
  \city{Stanford}
  \state{CA}
  \country{USA}
}

\begin{abstract}
A wide range of graph embedding objectives decompose into two components: one that enforces similarity, attracting the embeddings of nodes that are perceived as similar, and another that enforces dissimilarity, repelling the embeddings of nodes that are perceived as dissimilar. Without repulsion, the embeddings would collapse into trivial solutions. Skip-Gram Negative Sampling (SGNS) is a popular and efficient repulsion approach that prevents collapse by repelling each node from a sample of dissimilar nodes.
In this work, we show that when repulsion is most needed and the embeddings approach collapse, SGNS node-wise repulsion is, in the aggregate, an approximate re-centering of the node embedding dimensions. Such dimension operations are more scalable than node operations and produce a simpler geometric interpretation of the repulsion. Our theoretical result establishes dimension regularization as an effective and more efficient, compared to skip-gram node contrast, approach to enforcing dissimilarity among embeddings of nodes.
We use this result to propose a flexible algorithm augmentation framework that improves the scalability of any existing algorithm using SGNS. The framework prioritizes node attraction and replaces SGNS with dimension regularization. We instantiate this generic framework for LINE and node2vec and show that the augmented algorithms preserve downstream link-prediction performance while reducing GPU memory usage by up to $33.3\%$ and training time by $23.4\%$. Moreover, we show that completely removing repulsion (a special case of our augmentation framework) in LINE reduces training time by $70.9\%$ on average, while increasing link prediction performance, especially for graphs that are globally sparse but locally dense. Global sparsity slows down dimensional collapse, while local density ensures that node attraction brings the nodes near their neighbors. In general, however, repulsion is needed, and dimension regularization provides an efficient alternative to SGNS.
\end{abstract}

\maketitle

\section{Introduction}
Graph embedding algorithms use the structure of graphs to learn node-level embeddings. In unsupervised and supervised graph embedding algorithms, loss functions attempt to preserve \emph{similarity} and \emph{dissimilarity}. Nodes that are similar in the input graph should have similar embeddings, while dissimilar nodes should have dissimilar embeddings~\citep{bohm2022attraction, tsitsulin2018verse}. The push and pull of the similarity and dissimilarity objectives are key: in the absence of a dissimilarity objective, the loss would be minimized by embedding all nodes at a single embedding point, a degenerate and useless embedding. Often, enforcing dissimilarity is much more expensive than similarity, owing to the generally sparse nature of graphs and the number of pairs of dissimilar nodes growing quadratically with the size of the graph. Enforcing dissimilarity is also complex for graphs because graph data frequently have missing edges or noise~\citep{young2021bayesian, newman2018network}. Here, we show that while many past works have focused on repelling pairs of dissimilar nodes, the repulsion can be replaced with a regularization of the embedding dimensions, which is more scalable. 

The skip-gram (SG) model is one of the most popular approaches to graph embeddings~\citep{ahmed2013distributed, yang2024negative}. In the SG model, a target node is similar to a source node if it is in the source node's neighborhood. Furthermore, skip-gram negative sampling (SGNS) is a dominant method to efficiently approximate dissimilarity preservation. Instead of repelling all pairs of dissimilar nodes, SGNS repels only a sample of dissimilar nodes per pair of similar nodes. SGNS is utilized in LINE~\cite{tang2015LINE} and node2vec~\cite{grover2016node2vec}, for instance, and has been shown to yield strong downstream performance. However, several analytical issues with SGNS have also been identified. First, SGNS introduces a bias by re-scaling the relative importance of preserving similarity and dissimilarity~\citep{rudolph2016exponential}. Second, with SGNS, in the limit as the number of nodes in the graph approaches infinity, the similarities among embeddings diverge from the similarities among nodes in the graph~\citep{davison2023asymptotics}. Although SGNS has been used to learn both graph and word embeddings~\citep{mikolov2013distributed, mimno2017strange}, we focus on the graph context because, for graph embeddings in particular, SGNS remains a popular method for preserving dissimilarity~\citep{chami2022machine}. 

In this paper, we propose a change in perspective and show that node repulsion in the SG model can be achieved via dimension centering. If $X$ is an embedding matrix where the rows are node embeddings, ``dimensions'' refers to the columns of $X$. We draw inspiration from advances in the self-supervised learning (SSL) literature, which show an equivalence between sample-contrastive learning and dimension-contrastive learning~\citep{garrido2022duality, bardes2022vicreg}. Sample-contrastive learning explicitly repels dissimilar pairs while dimension-contrastive learning regularizes dimensions. 

The known parallels between sample and dimension contrast, however, do not suggest whether SG loss functions can also be re-interpreted from the dimension perspective. We present a novel theoretical contribution that extends the dimension-contrastive approach to a new class of loss functions. Our findings show that while the dissimilarity term in the SG loss is not itself a dimension regularizer when the term is most needed to counteract the attraction of similarity, the dissimilarity preservation can be achieved via regularization. We begin by characterizing the degenerate embedding behavior when the dissimilarity term is removed altogether. We prove that, under initialization conditions,\footnote{Our asymptotic result holds when the learning rate approaches zero and the embedding initializations approach the origin, an assumption inherited from the implicit regularization in matrix factorization established in~\citet{gunasekar2017implicit}.} when only positive pairs are considered, the embeddings collapse into a low-dimensional space. However, as the dimensions approach collapse, the dissimilarity term approaches a dimension-mean regularizer. 

We operationalize the dimension-based approach with an algorithm augmentation framework. We augment existing algorithms using SGNS by making two modifications. First, the augmentation framework prioritizes similarity preservation over dissimilarity preservation. This is desirable because, in real-world graph data, the lack of similarity between two nodes does not necessarily suggest that the two nodes are dissimilar. Second, when the embeddings begin to collapse after optimizing only for similarity preservation for a fixed number of epochs, our augmentation framework repels nodes from each other using a regularizer that induces embedding dimensions centered on the origin.

In summary, our contributions are as follows:
\begin{enumerate}
    \item In Section \ref{sec:regularization}, we map node repulsion to dimension regularization for the SG model. We show that instead of shortcutting the full skip-gram loss function with SGNS and repelling a sample of pairs, the repulsion function can be approximated with a dimension-mean regularization. We prove that as the need for node repulsion grows, optimizing the regularizer converges to optimizing the skip-gram loss. Our result extends the equivalence between sample-contrastive objectives and dimension regularization established in self-supervised learning to a new class of loss functions.
    \item In Section \ref{sec:algs}, we introduce a generic algorithm augmentation framework that prioritizes node attraction and replaces SGNS with occasional dimension regularization for any existing SG algorithm. We instantiate the augmentation framework for node2vec and LINE, reducing the repulsion complexity from $\mathcal{O}(n)$ to $\mathcal{O}(d)$ vector additions per epoch. (The number of embedding dimensions $d$ is much less than the number of nodes $n$.)
    \item In Section \ref{sec:eval}, we show that our augmentation framework reduces runtime and memory usage while also preserving link-prediction performance. Moreover, in globally sparse networks with high local density, removing repulsion altogether, a special case of our framework, even improves performance for LINE. However, in networks with low local density, repulsion is needed, and dimension regularization provides an efficient solution.
\end{enumerate}

\section{Related Works}
In this section, we review the popular use of SGNS within graph embeddings as well as its limitations. Our approach is also similar in spirit to the growing body of literature on non-contrastive learning and can be used in concert with existing scalable graph learning methods.

\subsection{Skip-Gram Negative Sampling}
SGNS was introduced in word2vec by \citet{mikolov2013distributed} as an efficient method for learning word embeddings. While the softmax normalization constant is costly to optimize, \citet{mikolov2013distributed} modeled SGNS after Noise Contrastive Estimation (NCE)~\cite{gutmann2012noise}, which learns to separate positive samples from samples drawn from a noise distribution. SGNS has since been adopted for graph representation learning, where it is utilized in both unsupervised~\citep{grover2016node2vec, tang2015LINE, perozzi2014deepwalk} and supervised~\citep{yang2016revisiting} skip-gram models.  


At the same time, there are many known limitations of SGNS. \citet{rudolph2016exponential} place SGNS embeddings within the framework of Exponential Family Embeddings and note that SGNS downweights the magnitude of the negative update and leads to biased embeddings, relative to the gradients of the non-sampled objective. Second, \citet{davison2023asymptotics} examine the limiting distribution of embeddings learned via SGNS and show that the distribution decouples from the true sampling distribution in the limit. Last, it has also been shown that the optimal noise distribution and the corresponding parameters can vary by dataset~\citep{yang2020understanding}.

We would also like to note that while the motivations are similar, SGNS differs from the negative sampling that has arisen in the self-supervised learning literature~\citep{robinson2021contrastive}. In self-supervised learning, the negative samples are generally other nodes in the training batch.

\subsection{Non-contrastive Self-Supervised Learning}
Energy-Based Models in self-supervised learning are a unified framework for balancing similarity and dissimilarity~\citep{lecun_energy-based_2007}. As in our decomposition, energy-based models ensure that similar pairs have low energy and dissimilar pairs have high energy. Within energy-based models, there has been more focus across both vision and graph representation learning on contrastive models, which explicitly repel dissimilar pairs~\citep{in2023similarity-contrastive, li2023homogcl, yang2023batchsampler, zhang2023contrastive}. However, given the computational complexities of pairwise contrast, there is a growing body of work on non-contrastive representation learning methods, which do not use negative samples. Embedding collapse is the main challenge facing non-contrastive methods, and several mitigation strategies have emerged, such as asymmetric encoders in SimSiam~\citep{chen2021exploring}, momentum encoders in BYOL~\citep{grill2020bootstrap}, and redundancy reduction in Barlow Twins~\citep{zbontar2021barlow}. \citet{garrido2022duality} establish a duality between dimension-regularization based non-contrastive approaches and standard contrastive learning; however, the work specifically analyzes the squared-loss term, and there are no existing works, to our knowledge, establishing a connection between skip-gram loss and non-contrastive methods.

\subsection{Scalable and Efficient Graph Learning}
Past works on developing scalable and efficient representation learning for large graphs have focused on \emph{data-driven} scalability, utilizing node/edge sampling~\cite{hamilton2017inductive, chen2018fastgcn, zeng2020graphsaint}, graph partitioning~\cite{chiang2019clustergcn, zhu2019graphvite}, and graph coarsening~\cite{liang2021mile}.
The sampling-based approaches mitigate the exponential growth of neighbors as the depth of message-passing layers increases in GNNs. One approach is to sample at the layer level, such as in GraphSAGE~\cite{hamilton2017inductive} and FastGCN~\cite{chen2018fastgcn}, where the neighbors of each node are sampled in each layer. An alternative approach, such as GraphSAINT~\cite{zeng2020graphsaint}, is to first sample a subgraph and then train a full GNN, without layer sampling, on the subgraph.
Likewise, partitioning-based approaches also prevent neighborhood explosion. For instance, ClusterGCN~\cite{chiang2019clustergcn} performs message passing within clusters to reduce neighborhood scope, while GraphVite~\cite{zhu2019graphvite} leverages partitioning to train in parallel. 
Finally, coarsening approaches leverage the multi-scale nature of graphs. For example, MILE~\cite{liang2021mile} first learns embeddings for the structural backbone of a graph, which is much smaller, before refining the backbone embeddings onto the full, original graph. Our approach is orthogonal to data-driven approaches to scalability and can be used in conjunction; for example, our dimensionality-regularization bypass improves the efficiency of learning skip-gram embeddings for graph backbones or cluster partitions.

Instead, our approach is more similar to scalability approaches that propose efficiency bypasses for loss functions or optimization processes. 
For instance, \citet{zhu2025transformers} identify that the outputs of normalization layers in transformers approximate the hyperbolic tangent function, and replacing normalization layers with element-wise $\tanh$ transformations preserves accuracy while reducing training time. Meanwhile, \citet{loveland2025understanding} observe that one purpose of negative sampling in collaborative filtering is to increase the rank of the embedding matrix; the authors show that replacing negative sampling with a regularization on matrix rank, which is more efficient, effectively reduces compute cost. In contrast to these two prior works proposing bypasses, our work contributes a theoretical result showing that as the number of nodes increases, the difference between the expensive operation, SGNS, and the efficiency bypass, dimension-mean regularization, vanishes.
\section{From Node Repulsion to Dimension Regularization}\label{sec:regularization}

In this section, we introduce a generic loss decomposition where a function $P$ operationalizes similarity preservation and a ``negative'' function $N$ achieves dissimilarity preservation. We then show that instead of optimizing negative functions with costly node repulsions, we can instead regularize dimensions. Crucially, in Section \ref{sec:sg-norm}, we show that when node repulsion is needed, the negative function in the skip-gram loss can be optimized via dimension regularization.

Using notation introduced in Table \ref{tab:notation}, the decomposition is as follows: given an embedding matrix $X \in \mathbb{R}^{n \times d}$ and a similarity matrix $S \in \{0, 1\}^{n \times n}$, where $S_{ij} = 1$ if nodes $i$ and $j$ are similar, a generic graph embedding loss function $L(X, S)$ can be written as:
\begin{equation}\label{eqn:decomposition}
    L(X, S) = P(X, S) + N(X, S).
\end{equation}

\begin{table}[t!]
    \centering
    \caption{Notations used in this paper.}
    \Description[Notation table]{One column of symbols on the left and another column with the corresponding meanings on the right.}
    \begin{tabular}{p{0.15\linewidth} p{0.7\linewidth}}
        \toprule
         \textbf{Symbol} &  \textbf{Meaning} \\ \midrule
         $G, V, E$ & Graph $G$ with nodes $V$ and edges $E$\\
         $n, m$ & number of nodes and edges respectively\\
         $d$ & number of embedding dimensions\\
         $N$, $P$ & negative and positive loss functions\\
         $S$ & similarity matrix $\in \{0, 1\}^{n \times n}$\\
         $X$ & node embedding matrix $\in \mathbb{R}^{n \times d}$\\
         $X_i$ & $i^{th}$ row of $X$, as a column vector\\
         $X_{.j}$ & $j^{th}$ column of $X$, as a column vector\\
         $P_\alpha$ & Probability distribution with parameter $\alpha$ \\
         $p, q$ & node2vec random-walk bias parameters \cite{grover2016node2vec} \\
         $k$ & number of negative samples per positive pair\\
         $b$ & average number of positive pairs per node\\
         $\eta$ & learning rate\\
         $D_x$ & diagonal matrix where $x$ is the diagonal\\
         $\mathcal{C}$ & the constriction of the embeddings (Def. \ref{def:constriction})\\
         $\Vec{1}, \mathbf{1}$ & a vector and a matrix of all ones, respectively\\
         $[d]$ & $\{1, 2, \dots, d\}$\\
        \bottomrule
    \end{tabular}

    \label{tab:notation}
\end{table}

The decomposition in Eq.~\eqref{eqn:decomposition} applies to nearly all unsupervised graph embedding objectives as well as many supervised learning objectives, where supervision is provided in the form of node labels. In the graph embedding survey by \citet{chami2022machine}, the decomposition applies to all unsupervised methods except for Graph Factorization~\citep{ahmed2013distributed}, which does not include a negative function $N$. Examples of popular decomposable loss functions are matrix reconstruction error (e.g., spectral embeddings) as well as softmax (e.g. node2vec~\citep{grover2016node2vec} and LINE~\citep{tang2015LINE}). The decomposition also applies to supervised methods that regularize for graph structure ($\beta > 0$ as defined in \citet{chami2022machine}), such as Neural Graph Machines~\citep{bui2018neural} and Planetoid~\citep{yang2016revisiting}. 

Given that graphs are sparse, performing gradient descent on $N$ is costly as the gradient update repels all dissimilar pairs, resulting in $\mathcal{O}\left(n^2\right)$ vector additions per epoch. In this paper, we build on the argument that the costly \emph{node-wise} operation can be replaced with a more efficient \emph{dimension-wise} operation, as summarized in Figure \ref{fig:tree}.

\begin{figure}[ht]
    \centering
    \includegraphics[width = 0.9\linewidth]{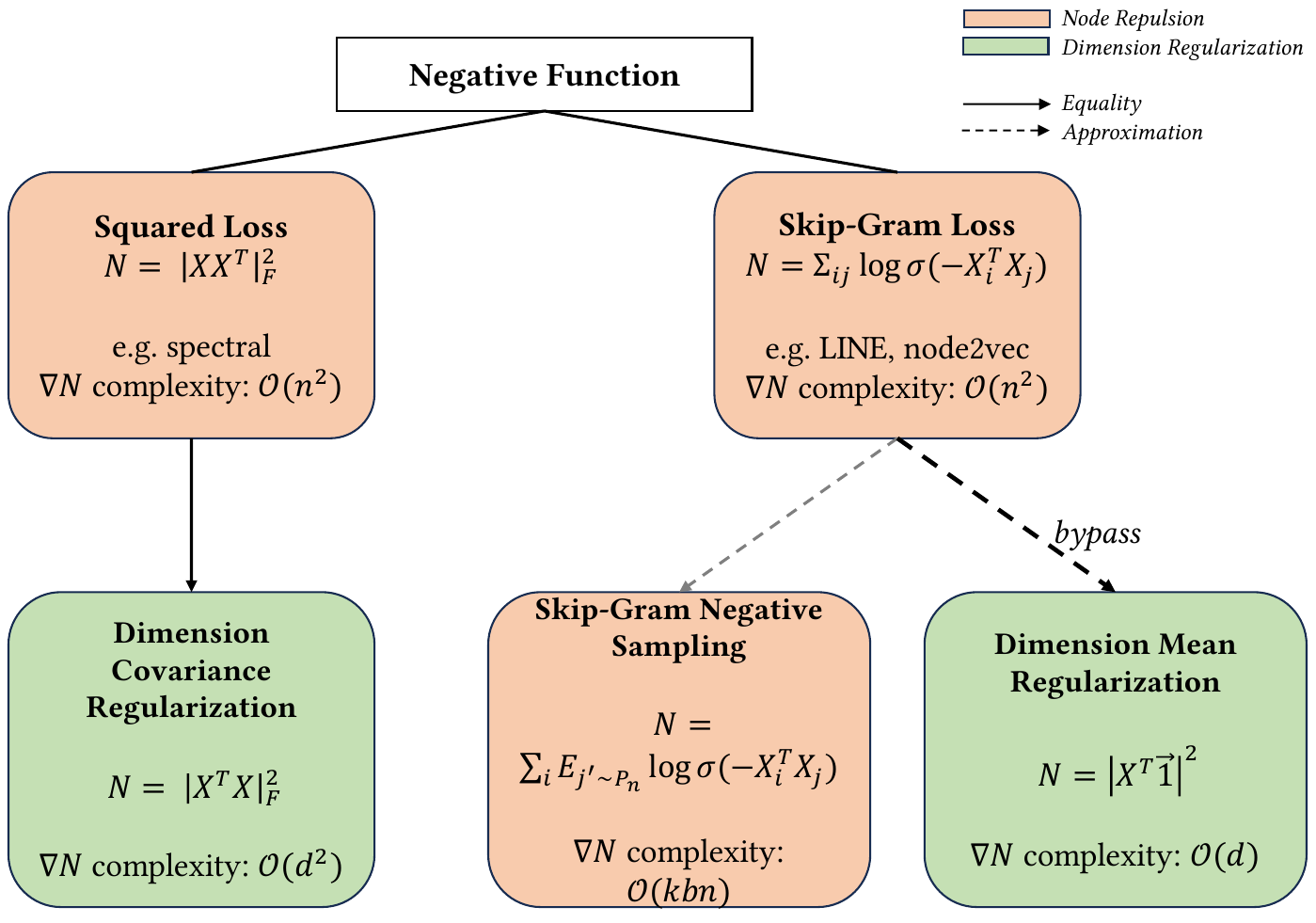}
    \Description[Conceptual tree]{A tree diagram summarizing our theoretical contribution. The root node is labeled "negative function" and branches into "squared loss" and "skip-gram loss". An example of squared loss is dimension-covariance regularization. With skip-gram loss, one descendant is SGNS while the other is our dimension-mean regularization. Each method is labeled with its computational complexity.}
    \caption{
        Nearly all unsupervised and many supervised graph embedding loss functions define a ``negative'' function that repels embeddings of dissimilar nodes. We show that instead of repelling pairs of nodes (orange), which is costly, the negative function in the popular skip-gram (SG) loss can be approximated with a dimension-mean regularizer. The regularizer is efficient given that $d \ll n$. Our result establishing a connection between the skip-gram loss and dimension-mean regularization is novel. Runtime complexities are expressed in terms of the number of vector additions per epoch with details in Section \ref{sec:reweighting}.
    }
    \label{fig:tree}
\end{figure}

Below, we will map two dominant negative functions found in graph embedding algorithms to dimension regularizations. The two loss functions are spectral loss functions and skip-gram loss functions. The mapping of spectral loss functions to dimension covariance regularization in Section \ref{sec:ase-norm} parallels recent approaches to non-contrastive learning~\citep{bardes2022vicreg, garrido2022duality}. Our novel contribution, detailed in Section \ref{sec:sg-norm}, is a mapping from the skip-gram loss to a regularizer that induces origin-centered dimensions. Proof for all propositions below are included in Appendix \ref{sec:proofs}.

\subsection{Dimension Regularization for Spectral Embeddings} \label{sec:ase-norm}

In the case of Adjacency Spectral Embeddings (ASE)~\citep{sussman2012consistent}, which are equivalent to taking the leading eigenvectors of the adjacency matrix, the matrix $S$ is the adjacency matrix $A \in \{0, 1\}^{n \times n}$. 
\begin{align}
    L_{ASE}(X, S) &= \| S - XX^T \|_2^2,\\
    P_{ASE}(X, S) &= -2 \sum_{i,j \mid S_{ij} = 1} X_i^T X_j + \|S\|_F^2,\\
    N_{ASE}(X, S) &= \| XX^T \|_F^2.
\end{align}

On one hand, performing gradient descent on $N_{ASE}$ can be interpreted as repelling all pairs of embeddings where the repulsion magnitude is the dot product between embeddings. If $\eta$ is a learning rate and $t$ is the step count, the embedding for node $i$ is updated as:
\begin{equation}\label{eqn:ase-sample-contrast}
    X_i^{t+1} = X_i^t - \eta \sum_{i' \in V} \left(X_i^TX_{i'}\right) X_{i'} ~~ \forall i \in V.
\end{equation}
The same negative function can also be written as a dimension covariance regularization:
\begin{proposition}\label{prop:ase-regularization}
    $N_{ASE}$ is equivalent to the regularization function $\| X^TX\|_F^2$ which penalizes covariance among dimensions.
\end{proposition}

With Proposition \ref{prop:ase-regularization}, we can re-interpret the gradient descent updates in Eq.~\eqref{eqn:ase-sample-contrast} as collectively repelling dimensions. The gradient update can now be written in terms of dimensions:
\begin{equation}
    X_{.j}^{t+1} = X_{.j}^t - \eta \sum_{j' \in [d]} \left(X_{.j}^TX_{.j'}\right) X_{.j'} ~~ \forall j \in [d].
\end{equation}

\subsection{Dimension Regularization for Skip-Gram Embeddings} \label{sec:sg-norm}
We now introduce a novel dimension-based approach for skip-gram embeddings. For skip-gram embeddings, the similarity matrix is defined such that $S_{ij} = 1$ if node $j$ is in the \emph{neighborhood} of $i$. For first-order LINE, the neighborhood for node $i$ is simply all nodes connected to $i$, while for node2vec, the neighborhood is defined as all nodes within the context of $i$ on a random walk. The skip-gram (SG) loss function can be decomposed as:
\begin{align}
    L_{SG}(X, S) &= - \sum_{i,j} S_{ij} \log \sigma\left(X_i^T X_j\right) + (1 - S_{ij}) \log \sigma\left(- X_i^T X_j\right),\\ \label{eqn:sg-loss}
    P_{SG}(X, S) &= - \sum_{i,j \mid S_{ij} = 1} \log \sigma\left(X_i^T X_j\right), \\
    N_{SG}(X, S) &= - \sum_{i,j \mid S_{ij} = 0} \log \sigma\left(- X_i^T X_j\right). \label{eqn:sg-negative}
\end{align}

Our goal is to map $N_{SG}$ to a dimension regularization. 
Recall that this work is motivated by the fact that the purpose of $N_{SG}$ is to prevent the similarity $\sigma\left(X_i^TX_j\right)$ from increasing for all $i, j$; without $N_{SG}$, trivial embedding solutions can emerge that maximize similarity for all pairs of nodes, not just similar pairs. 

\paragraph{Guaranteed collapse.} To measure the onset of the degenerate condition in which all pairs of nodes are similar, we define the constriction $\mathcal{C}$ of a set of embeddings to be the minimum dot product between any pair of nodes: 
\begin{definition}[Constriction] The constriction $\mathcal{C}$ of an embedding matrix $X$ is defined as:
$\mathcal{C} = \min_{i,j \in [n]\times[n]} X_i^TX_j$.
\label{def:constriction}
\end{definition}

Geometrically, the embedding constriction is maximized when embeddings are radially squeezed and growing in magnitude, that is, collapsed. 
Proposition \ref{prop:guaranteed-oversmoothing} states that if we remove $N_{SG}$ altogether and the embeddings are initialized with sufficiently small norm and learning rate, the degenerate collapse will inevitably arise during gradient descent. 

\begin{proposition}\label{prop:guaranteed-oversmoothing}
    As the Euclidean norm of the initial embeddings and the learning rate approach zero, then in the process of performing gradient descent on $P_{SG}$, there exists a step $t$ such that for all gradient updates after $t$, the constriction $\mathcal{C}$ is positive and monotonically increasing.
\end{proposition}

The proof sketch for Proposition \ref{prop:guaranteed-oversmoothing} is as follows: as the embeddings are initialized closer to the origin, gradient descent on $P_{SG}$ approaches gradient descent on the matrix completion loss function: $\| \mathbf{1}_{S > 0} \odot \left(\mathbf{1} - \frac{1}{2}XX^T\right)\|_F^2$, where $\mathbf{1}_{S > 0}$ is the indicator matrix for whether entries of $S$ are positive. From \citet{gunasekar2017implicit}, gradient descent implicitly regularizes matrix completion to converge to the minimum nuclear solution; this implicit regularization drives all dot products to be positive, not just pairs of embeddings corresponding to connected nodes.

\paragraph{Approaching dimension regularization.} Now, we show that as constriction increases, performing gradient descent on $N_{SG}$ approaches optimizing a dimension regularizer. That is, when repulsion is most needed and the embeddings approach collapse due to similarity preservation, repulsion can be achieved via regularization.

First, we map $N_{SG}$ to an ``all-to-all'' node repulsion. While $N_{SG}$ only sums over negative node pairs ($i, j$ where $j$ is \emph{not} in the neighborhood of $i$), for large, sparse graphs we can approximate $N_{SG}$ with the objective $N'_{SG}$ which sums over all pairs of nodes:
\begin{equation}\label{eqn:sg-all-to-all}
    N_{SG}' = - \sum_{i,j} \log \sigma\left(-X_i^TX_j\right).
\end{equation}

Proposition \ref{prop:all-to-all-approx} states that if the embedding norms are bounded and the constriction $\mathcal{C} > 0$, then, in the limit of $n$, the difference between the gradient of $N'_{SG}$ and $N_{SG}$ approaches zero. 

\begin{proposition} \label{prop:all-to-all-approx}
If all embeddings have norms that are neither infinitely large or vanishingly small and the embedding constriction $\mathcal{C} > 0$, then, as the number of nodes in a sparse graph grows to infinity, the gradients of $\nabla N_{SG}$ and $\nabla N'_{SG}$ converge:
\begin{equation}
\lim_{n \to \infty} \frac{\|\nabla N'_{SG} - \nabla N_{SG}\|^2_F}{\|\nabla N_{SG}\|^2_F} = 0,
\end{equation}
where a graph is sparse if $\lvert E\lvert$ is $o(n^2)$.
\end{proposition}

For a single node $i$, performing gradient descent on Eq.~\eqref{eqn:sg-all-to-all} results in the following update:
\begin{equation}\label{eqn:sg-gradient}
    X_{i}^{t+1} = 
    X_i^t - \eta\sum_{i' \in V} \sigma\left(\left(X_i^t\right)^T X_{i'}^t\right) X_{i'}^t .
\end{equation}

The right-hand term in the gradient update repels node $i$ from all other nodes where the repulsion is proportional to the similarity between the node embeddings. In Proposition \ref{prop:sg-regularization}, we show a connection between minimizing $N'_{SG}$ and centering the dimensions at the origin. For intuition, observe that if all pairs of nodes are highly similar, the gradient update in Eq.~\eqref{eqn:sg-gradient} is approximately equal to subtracting the column means scaled by a constant ($2\eta \left(X^T\Vec{1}\right)$). This is equivalent to performing gradient descent on a dimension regularizer that penalizes non-zero dimension means,
\begin{equation} \label{eqn:mean-regularizer}
    R(X) = \| X^T \Vec{1}\|^2_2.
\end{equation}
We formalize the connection between the negative function $N'_{SG}$ and origin-centering in the following proposition:

\begin{proposition} \label{prop:sg-regularization}
    Let $R$ be the dimension regularizer defined in Eq.~\eqref{eqn:mean-regularizer} that penalizes embeddings centered away from the origin and $n \gg d$. Then, as the constriction increases beyond zero, the difference between performing gradient descent on $R$ versus $N'_{SG}$ vanishes.
\end{proposition}

We note that our result establishing a connection between the skip-gram loss and origin-centered dimensions is related to the finding in \citet{wang2022understanding} connecting the InfoNCE loss with embeddings uniformly distributed on unit hyperspheres; however, as the SG and InfoNCE losses are distinct, our theoretical contribution in Proposition \ref{prop:sg-regularization} is novel.

\subsubsection{Comparison with Skip-Gram Negative Sampling}\label{sec:reweighting}
Skip-gram negative sampling (SGNS) offers an efficient stochastic approximation to the gradient update in Eq.~\eqref{eqn:sg-gradient}. Furthermore, the SGNS procedure provides a tunable way to bias the gradients---via non-uniform sampling---in a manner that has been seen to empirically improve the utility of the resulting embedding in downstream tasks~\citep{mikolov2013distributed}. Instead of repelling node $i$ from all other $n - 1$ nodes, SGNS repels $i$ from a sample of $k$ nodes where the nodes are sampled according to a distribution $P_\alpha$ over all nodes, optimizing the following objective:
\begin{equation}
    N_{SGNS}(X, S, i) = - k \mathbb{E}_{j' \sim P_\alpha} \left[ \log \sigma\left(-X_i^T X_{j'}\right)\right],
\end{equation}
where the expectation is estimated based on $k$ samples. 

In aggregate, SGNS reduces the repulsion time complexity from $\mathcal{O}\left(n^2\right)$ to $\mathcal{O}\left(kbn\right)$ vector additions per epoch. $b$ is the average number of positive pairs per node; for LINE, $b=m/n$, and for node2vec, $b$ is the product of the context size and random-walk length. With Proposition \ref{prop:sg-regularization}, we reduce the time complexity to $\mathcal{O}(d)$ vector additions per epoch.

As mentioned, SGNS embeddings can be tuned by the choice of the non-uniform sampling distribution, where in graph embedding contexts the distribution $P_\alpha$ is typically sampling nodes proportional to their degree$^\alpha$, with $\alpha = 3/4$. An optimization-based intuition for this choice is that a degree-based non-uniform distribution prioritizes learning the embeddings of high-degree nodes, but we emphasize that the specific choice of $\alpha=3/4$ is typically motivated directly based on improved empirical performance in downstream tasks \cite{mikolov2013distributed}. 

We briefly note that our dimension-mean regularizer is immediately amenable to introducing an analogous tuning opportunity. We can simply replace the regularization in Eq.~\eqref{eqn:mean-regularizer} with 
\begin{equation} \label{eqn:mean-regularizer-weighted}
    R(X;\Vec{p}) = \| X^T \Vec{p}\|^2_2,
\end{equation}
where $\Vec{p}$ is a normalized weight vector that biases the negative update in exact correspondence to the probabilities of each node in $P_\alpha$. In our subsequent simulations, we focus our efforts on the uniform case of $\Vec{p}=\Vec{1}$, i.e., the regularizer in Eq.~\eqref{eqn:mean-regularizer}.
\section{Algorithm Augmentation to Replace SGNS} \label{sec:algs}
We now propose a generic algorithm augmentation framework that directly replaces SGNS with dimension regularization. We instantiate this algorithm augmentation for LINE and node2vec, but note that the framework is applicable to any graph embedding algorithm using SGNS.

Our augmentation modifies existing algorithms using SGNS in two ways. First, the augmentation prioritizes the positive function $P$, that is, preserving similarity when possible. Not only does prioritizing observed edges increase efficiency, but it is also desirable given the fact that real-world graph data frequently have missing edges. We prioritize the observed edges by defaulting to performing gradient descent on $P$.

Second, our augmentation achieves repulsion via gradient descent on $R$, the dimension-mean regularizer introduced in Eq.~\eqref{eqn:mean-regularizer}. As constriction increases from optimizing only $P$, dimension-mean regularization approximates $\nabla N'_{SG}$ per Proposition \ref{prop:sg-regularization}. 

Taken together, the algorithm augmentation framework can be summarized as:
\begin{equation}
\label{eqn:augmentation-framework}
X^{t + 1} = 
    \begin{cases}
        X^t - \eta \nabla P_{SG}\left(X^t\right) & t~\%~n_{\text{negative}} \ne 0,\\
        X^t - \eta \left[\nabla P_{SG}\left(X^t\right) + \lambda \nabla R\left(X^t\right)\right] & t~\%~n_{\text{negative}} = 0,
    \end{cases}
\end{equation}
where $\lambda$ is the regularization hyperparameter, $n_{\text{negative}}$ controls the frequency of performing gradient descent on $R$, and $\eta$ is the learning rate.

Figure \ref{fig:side-by-side} visualizes our algorithm augmentation with a Stochastic Block Model (SBM). We begin by randomly initializing embeddings around the origin and then repeatedly (over 100 epochs) attract the embeddings of similar nodes. The attraction drives the embeddings away from the origin and causes the two blocks to remain entangled. Then, we repel dissimilar pairs via dimension-mean regularization, which pulls the embeddings back to the origin. Once we resume applying positive updates, the embeddings for the two blocks begin to separate, indicating effective repulsion.
\begin{figure*}[ht]
    \centering
    \includegraphics[width = 0.95\linewidth]{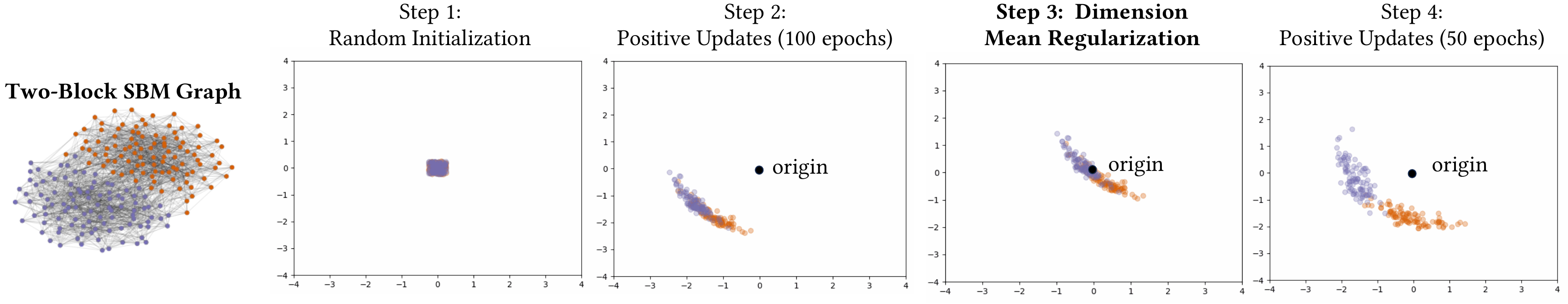}
    \Description[Toy example of our algorithm augmentation]{On the left is a toy two-block SBM graph. The right side of the figure shows the embedding of the SBM graph at four snapshots. The snapshots show: 1) random initialization at the origin 2) approaching collapse away from the origin after positive updates 3) retraction to the origin from dimension-mean regularization 4) expansion from the origin but with the two blocks separated in embedding space.}
    \caption{
        We use a two-block Stochastic Block Model (SBM) example to summarize how we bypass skip-gram negative sampling with dimension regularization. We introduce an algorithm augmentation that prioritizes attracting embeddings of similar nodes together (Step 2), where similarity is the sigmoid of embedding dot products. Eventually, to prevent all pairs of embeddings from becoming similar, the dimension regularization re-centers the embeddings around the origin, increasing node contrast (Step 3). Attraction updates are then again repeatedly applied, but the blocks are more distinguishable following dimension regularization (Step 4). 
    }
    \label{fig:side-by-side}
\end{figure*}

\subsection{Instantiation for LINE and node2vec}

In Algorithm~\ref{alg:aug-framework} we include the pseudo-code for the augmented versions of LINE and node2vec, utilizing the framework in Eq.~\eqref{eqn:augmentation-framework}. The loop labeled ``old negative update'' shows where SGNS would have occurred. When instantiated for LINE and node2vec, we implement our algorithm augmentation with batched gradient descent, where batches are sets of positive pairs. Dimension regularization is applied every $n_{\text{negative}}$ batches, and the regularization is applied to the entire embedding matrix $X$.
\begin{algorithm}
\caption{Augmented LINE and node2vec}
\label{alg:aug-framework}
\begin{algorithmic}
\State {\bfseries Input:} $G, n, d, p, q, \text{num\_epochs}, \text{batch\_size}, \lambda, \eta, n_{\text{negative}}$
\State $t \gets 0$
\State $X^0 \gets \texttt{random\_initialization}(n, d)$
\State $\text{pairs} \gets \texttt{run\_random\_walks}(G, p, q)$ 
\For{$\{1, \dots, \text{num\_epochs}\}$}
    \State $\text{pair\_batches} \gets \texttt{get\_batches}(\text{pairs}, \text{batch\_size})$
    \For{batch \texttt{in} pair\_batches}
        \State $X^{t + 1} \gets X^t$
        \For{$i,j \in$ batch}
            \State $X^{t+1}_i \mathrel{+}=~ \eta \sigma\left(-\langle X^t_i, X^t_j\rangle \right)X^t_j$ 
            \Comment{positive update}
            \State $X^{t+1}_j \mathrel{+}=~ \eta \sigma\left(-\langle X^t_i, X^t_j\rangle \right)X^t_i$
                \For{$j' \in \texttt{sample}\left(P_\alpha, k\right)$}
                    \Comment{old negative update}
                    \State $X^{t+1}_i \mathrel{-}=~ \eta \sigma\left(-\langle X^t_i, X^t_{j'}\rangle \right)X^t_{j'}$ 
                    \State $X^{t+1}_{j'} \mathrel{-}=~ \eta \sigma\left(-\langle X^t_i, X^t_{j'}\rangle \right)X^t_i$ 
                \EndFor
        \EndFor
        \State $t \gets t + 1$
    \EndFor
    \If{$t ~~ \% ~~ n_{\text{negative}} == 0$} 
    \Comment{new negative update}
        \State $X^{t+1} \mathrel{-}=~ \frac{\lambda}{n}\mathbf{1}X^t$
    \EndIf
\EndFor
\end{algorithmic}
\end{algorithm}

LINE and node2vec differ in their implementation of the random-walk generation function. 
For LINE, the function simply returns the edge set $E$. For node2vec, the function returns the set of all node and neighbor pairs ($\left\{i, j | S_{ij} = 1 \right\}$) where neighbors of $i$ are nodes that co-occur in a biased-random-walk context window. The parameters $p, q$ control the bias of the random walk \cite{grover2016node2vec}, interpolating between ``outward'' DFS-like walks and ``inward'' BFS-like walks. The default values for $p$ and $q$ are $1$; smaller values of $p$ increase the probability of backtracking to the previous node in the random walk and smaller values of $q$ encourage outward DFS-like walks. In Section \ref{sec:methodology}, we discuss our hyperparameter optimization methodology for $p$ and $q$.
\section{Experiments} \label{sec:eval} 
To assess the efficacy of our dimension-regularization augmentation framework for capturing topological structure, we evaluate the link-prediction performance of the node2vec and LINE instantiations. Our results show that replacing SGNS with dimension-mean regularization reduces both training time and GPU memory usage while preserving link-prediction performance. Moreover, for real-world networks, we show that applying no repulsion at all, a special case of our framework, nearly always outperforms LINE and reduces training time by $71.7\%$ on average. We correlate the success of attraction-only models with \emph{local density}, which is common in real-world networks. However, more generally, in networks with \emph{local sparsity}, repulsion is needed, and dimension regularization provides an efficient alternative to SGNS.

\subsection{Methodology}\label{sec:methodology}
Below, we summarize the data and experimental methodology used for our empirical evaluation. All of the experiments were executed on a machine with a single NVIDIA V100 32 GB GPU. Our code is available at: \href{https://doi.org/10.5281/zenodo.15557847}{https://doi.org/10.5281/zenodo.15557847}. More details are provided in Appendix \ref{sec:appendix-methodology}.

\paragraph{Data.} We utilize the popular link-prediction benchmarks of Cora, CiteSeer, and PubMed~\citep{yang2016revisiting} as well as four OGB datasets to demonstrate the scalability of our augmentation: ogbl-collab, ogbl-ppa, ogbl-citation2, ogbl-vessel ~\citep{hu2020ogb}. We uniformly-randomly split the edge sets for Cora, CiteSeer, and PubMed into training/validation/test sets following standard 70/10/20 splits. For the OGB datasets, we use the provided edge splits. Table \ref{tab:datasets} summarizes the structural properties of the datasets, showing that the datasets span two types of densities: global density, measured by edge density, and local density, measured by the average clustering coefficient (``CC'') \cite{watts1998collective}.

\begin{table}[h]
    \centering
    \caption{Summary of datasets with structural properties. ``CC'' refers to the average clustering coefficient.}
    \begin{tabular}{lrrrl}
        \toprule
        \textbf{Dataset} & \textbf{Nodes} & \textbf{Edges} & \textbf{CC} & \textbf{Edge Density} \\
        \midrule
        \textbf{Cora} & $2.7 \times 10^{3}$ & $1.2 \times 10^{4}$ & 0.24 & $1.6 \times 10^{-3}$ \\
        \textbf{CiteSeer} & $3.3 \times 10^{3}$ & $1.0 \times 10^{4}$ & 0.14 & $9.0 \times 10^{-4}$ \\
        \textbf{PubMed} & $2.0 \times 10^{4}$ & $9.8 \times 10^{4}$ & 0.06 & $2.5 \times 10^{-4}$ \\
        \textbf{ogbl-collab} & $2.4 \times 10^{5}$ & $1.3 \times 10^{6}$ & 0.73 & $2.4 \times 10^{-5}$ \\
        \textbf{ogbl-ppa} & $5.8 \times 10^{5}$ & $3.6 \times 10^{7}$ & 0.22 & $1.1 \times 10^{-4}$ \\
        \textbf{ogbl-citation2} & $2.9 \times 10^{6}$ & $3.1 \times 10^{7}$ & 0.09 & $3.6 \times 10^{-6}$ \\
        \textbf{ogbl-vessel} & $3.5 \times 10^{6}$ & $5.9 \times 10^{6}$ & 0.01 & $4.7 \times 10^{-7}$ \\
        \bottomrule
    \end{tabular}
    \label{tab:datasets}
\end{table}

\paragraph{Models.} For each of node2vec and LINE, we instantiate three model variants: variant \texttt{I} is vanilla node2vec/LINE; variant \texttt{II}$^0$ is a special instance of our framework in which no repulsion is applied at all ($n_{\text{negative}}$ is set to be larger than the number of batches); and variant \texttt{II} is our augmented model in which dimension regularization is applied at least once per epoch. The variant \texttt{II}$^0$ is trained for only 2 epochs to prevent collapse. For completeness, we also instantiate variants of the vanilla algorithms and our augmentation in which node repulsion is proportional to $\text{degree}^\alpha$ as discussed in Section \ref{sec:reweighting}; these results are comparable to the results for \texttt{I} and \texttt{II} and included in Appendix \ref{sec:supplemental-eval}. In all cases, we set the embedding dimension $d=128$.

For link prediction, the embeddings for each model are used to train a multi-layer perceptron (MLP) edge classifier. A MLP is trained for each combination of dataset and model variant (\texttt{I}, \texttt{II}$^0$, \texttt{II}). The input to MLP is the concatenation of two node embeddings and the output is an edge prediction score $\in [0, 1]$. The time and memory values in the results correspond only to the resources used to train the embeddings and exclude the MLP training.

\paragraph{Hyperparameter optimization.} For each dataset and SG model, we optimize the following hyperparameters: the learning rate $\eta$, the node2vec random-walk parameters $p$ and $q$, as well as $n_{\text{negative}}$ and $\lambda$ from our augmentation. Details about the optimization process and the final values are included in Appendix \ref{sec:appendix-methodology}.

\paragraph{Metrics.} We report AUC-ROC, MRR, and Hits@$k$ evaluation metrics for link prediction. All metrics are evaluated at the node level and then averaged. To account for class imbalance when measuring AUC-ROC, we sample negative edges such that the number of positive and negative test edges is equal.

\subsection{Results}
\subsubsection{Dimension regularization reduces training time and memory while preserving link-prediction performance.} Tables \ref{table:line-metrics} and \ref{table:node2vec-metrics} show the training time, GPU memory usage, and AUC-ROC for vanilla and augmented node2vec and LINE. The training time is aggregated over multiple epochs, where the number of epochs is fixed for each graph and specified in Appendix \ref{sec:appendix-methodology}. The reported GPU memory is the maximum GPU memory allocated by PyTorch over the course of embedding training. The $\Delta (\%)$ columns under time and memory report the relative difference between $\texttt{II}$ and $\texttt{I}$, which lower-bounds the efficiency improvement given that \texttt{II}$^0$, a specific instance of our framework, is the most efficient variant.

For LINE, replacing SGNS with dimension regularization decreases training time by $11.4\%$, and link-prediction performance increases for all graphs, with AUC-ROC increasing by $0.05$ on average. Of note, LINE does not learn meaningful embeddings for the ogbl-vessel graph, likely due to the graph's low edge density. For node2vec, training time reduces by $16.5\%$, on average. The AUC-ROC decreases for certain graphs, but, on average, the AUC-ROC changes by less than $0.05$. Meanwhile, GPU memory usage decreases by $30.7\%$ on average. In cases where GPU memory is constrained, our augmentation can enable node2vec training for graphs that were previously infeasible to train.

Our augmentation does not reduce memory usage for LINE because we obtain the best performance across all models when using small batch sizes; with small batches, even the vanilla algorithm has low GPU memory usage, and negative sampling contributes marginally to memory usage. We include the performance according to ranking metrics in Appendix \ref{sec:supplemental-eval}, which shows similar patterns as the AUC-ROC results.

\begin{table*}[ht]
 \caption{For LINE, our dimension-regularization augmentation framework decreases training time by $11.4\%$ on average. The differences in link prediction performance (as measured by AUC-ROC) are negligible ($\pm 0.05$ on average). We use small batch sizes for LINE; thus, memory use is not affected.
The column \texttt{I} refers to vanilla LINE. The column \texttt{II} refers to augmented LINE where dimension regularization is applied at least once per epoch. The column \texttt{II}$^0$ refers to a special case of our framework where no repulsion is applied. For training time and memory, the $\Delta$ columns report the relative difference between \texttt{II} and \texttt{I}, which lower-bounds the efficiency gain of our framework as \texttt{II}$^0$ is the most efficient.}
\centering
    \begin{tabular}{l|cccc|cccc|ccc}
    \hline
    \multirow{2}{*}{\textbf{Dataset}} & \multicolumn{4}{c|}{\textbf{Time (min)}} & \multicolumn{4}{c|}{\textbf{Memory (GB)}} & \multicolumn{3}{c}{\textbf{AUC-ROC}} \\
     & \texttt{I} & \texttt{II}$^0$ & \texttt{II} & \textbf{$\Delta (\%)$} & \texttt{I} & \texttt{II}$^0$ & \texttt{II} & \textbf{$\Delta (\%)$} & \texttt{I} & \texttt{II}$^0$ & \texttt{II}\\
    \hline
\textbf{CiteSeer} & 0.05 & <0.01 & 0.04 & -20.00\% & 0.02 & 0.02 & 0.02 & 0.00\% & 0.6 & 0.65 & 0.67\\
\textbf{Cora} & 0.06 & 0.01 & 0.05 & -16.67\% & 0.02 & 0.02 & 0.02 & 0.00\% & 0.59 & 0.67 & 0.68\\
\textbf{PubMed} & 0.55 & 0.05 & 0.45 & -18.18\% & 0.06 & 0.06 & 0.06 & 0.00\% & 0.7 & 0.72 & 0.75\\
\textbf{ogbl-collab} & 0.77 & 0.25 & 0.65 & -15.58\% & 0.64 & 0.64 & 0.64 & 0.00\% & 0.72 & 0.78 & 0.8\\
\textbf{ogbl-ppa} & 16.23 & 5.12 & 15.43 & -4.93\% & 2.07 & 2.07 & 2.07 & 0.00\% & 0.9 & 0.91 & 0.78\\
\textbf{ogbl-citation2} & 48.51 & 25.19 & 47.4 & -2.29\% & 8.02 & 8.02 & 8.02 & 0.00\% & 0.61 & 0.65 & 0.62\\
\textbf{ogbl-vessel} & 9.39 & 4.9 & 9.2 & -2.02\% & 9.19 & 9.19 & 9.19 & 0.00\% & 0.5 & 0.5 & 0.5\\
    \hline
    \end{tabular}
\label{table:line-metrics}
\end{table*}

\begin{table*}[ht]
\caption{For node2vec, our augmentation framework reduces training time by $16.5\%$ and memory by $30.7\%$ on average. The differences in link prediction performance (as measured by AUC-ROC) are negligible ($\pm 0.05$ on average). The columns are defined analogously as in Table \ref{table:line-metrics}.}
\centering
    \begin{tabular}{l|cccc|cccc|ccc}
    \hline
    \multirow{2}{*}{\textbf{Dataset}} & \multicolumn{4}{c|}{\textbf{Time (min)}} & \multicolumn{4}{c|}{\textbf{Memory (GB)}} & \multicolumn{3}{c}{\textbf{AUC-ROC}} \\
     & \texttt{I} & \texttt{II}$^0$ & \texttt{II} & \textbf{$\Delta (\%)$} & \texttt{I} & \texttt{II}$^0$ & \texttt{II} & \textbf{$\Delta (\%)$} & \texttt{I} & \texttt{II}$^0$ & \texttt{II}\\
    \hline
\textbf{CiteSeer}  & 0.53 & 0.03 & 0.42 & -20.75\% & 3.12 & 2.09 & 2.09 & -33.01\% & 0.76 & 0.78 & 0.77\\
\textbf{Cora}  & 0.44 & 0.03 & 0.36 & -18.18\% & 3.12 & 2.08 & 2.08 & -33.33\% & 0.87 & 0.66 & 0.8\\
\textbf{PubMed}  & 2.61 & 0.19 & 2.00 & -23.37\% & 3.16 & 2.12 & 2.12 & -32.91\% & 0.91 & 0.76 & 0.81\\
\textbf{ogbl-collab}  & 4.31 & 1.48 & 3.88 & -9.98\% & 13.25 & 9.00 & 9.00 & -32.08\% & 0.94 & 0.93 & 0.96\\
\textbf{ogbl-ppa}  & 24.72 & 9.23 & 23.43 & -5.22\% & 27.24 & 18.74 & 18.74 & -31.20\% & 1.00 & 0.81 & 0.97\\
\textbf{ogbl-citation2}  & 35.75 & 20.5 & 28.07 & -21.48\% & 31.98 & 23.49 & 23.49 & -26.55\% & 0.98 & 0.81 & 0.95\\
\textbf{ogbl-vessel}  & 51.69 & 27.93 & 42.99 & -16.83\% & 32.84 & 24.35 & 24.35 & -25.85\% & 0.96 & 0.83 & 0.87\\
    \hline
    \end{tabular}
\label{table:node2vec-metrics}
\end{table*}

\subsubsection{Removing repulsion altogether reduces training time by over $70\%$ and improves LINE's performance, especially for locally dense graphs.} Interestingly, we observe that the special case of our algorithm augmentation that removes repulsion altogether (\texttt{II}$^0$) outperforms vanilla LINE for all seven graphs and significantly reduces training time. The top of Figure \ref{fig:clustering} shows that removing repulsion reduces training time by at least $40\%$. Meanwhile, the bottom of the figure correlates the gain in LINE's performance with local edge density, as measured by the average clustering coefficient. The trend is most noticeable in the case of ogbl-collab, which has the highest clustering coefficient and reflects increases, relative to vanilla LINE, in all three metrics when repulsion is removed. 

Intuitively, removing repulsion is effective in globally sparse but locally dense graphs, which characterizes many real-world networks~\citep{newman2003structure}, because the global sparsity hinders dimensional collapse while the local density ensures that node attraction brings nodes near hidden neighbors. While we would theoretically expect the embeddings to collapse even in sparse networks, in our experiments, we use the Adam optimizer, which decays the learning rate, and we train for only $2$ epochs in the case of \texttt{II}$^0$. On the other hand, removing repulsion is less effective for node2vec because the random walks effectively increase global density--and thus, the likelihood of collapse--by creating positive pairs among not just connected nodes but also nodes within the same context window.

\begin{figure}[h]
    \centering
    \begin{subfigure}[b]{0.8\columnwidth}
    \includegraphics[width=\columnwidth]{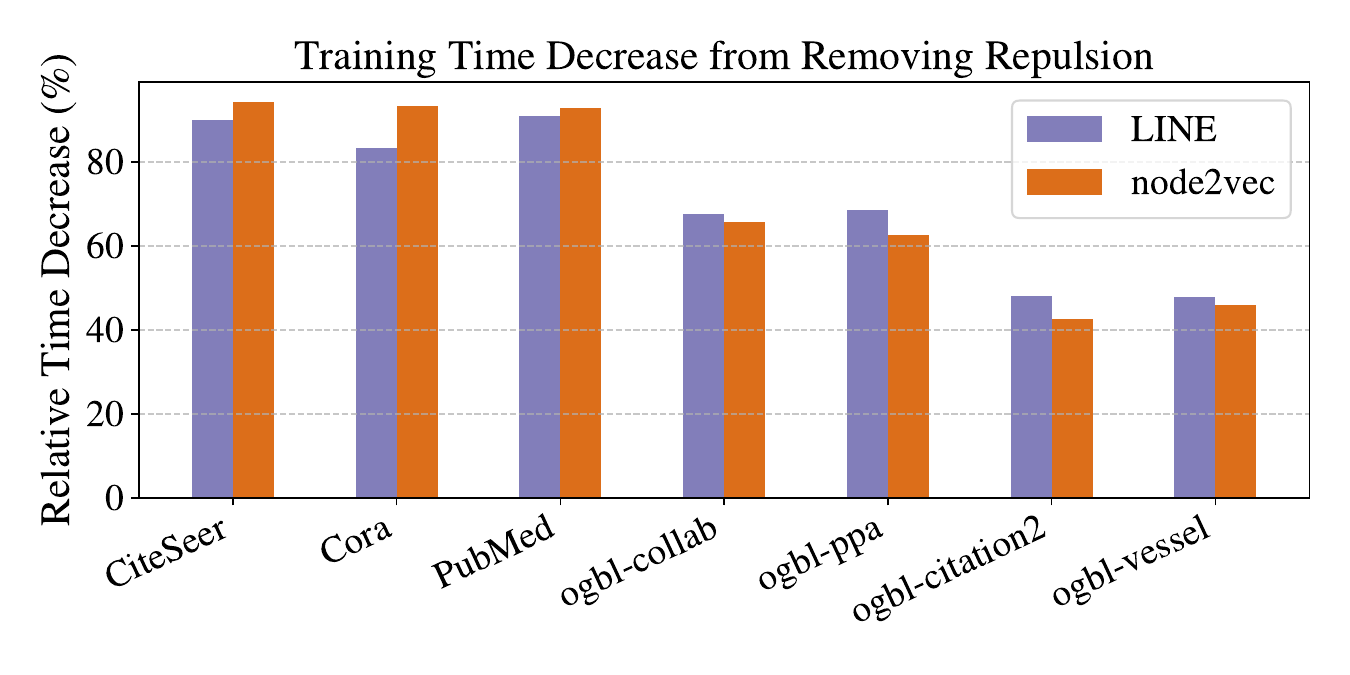}
    \end{subfigure}
    \begin{subfigure}[b]{\columnwidth}
    \includegraphics[width=\columnwidth]{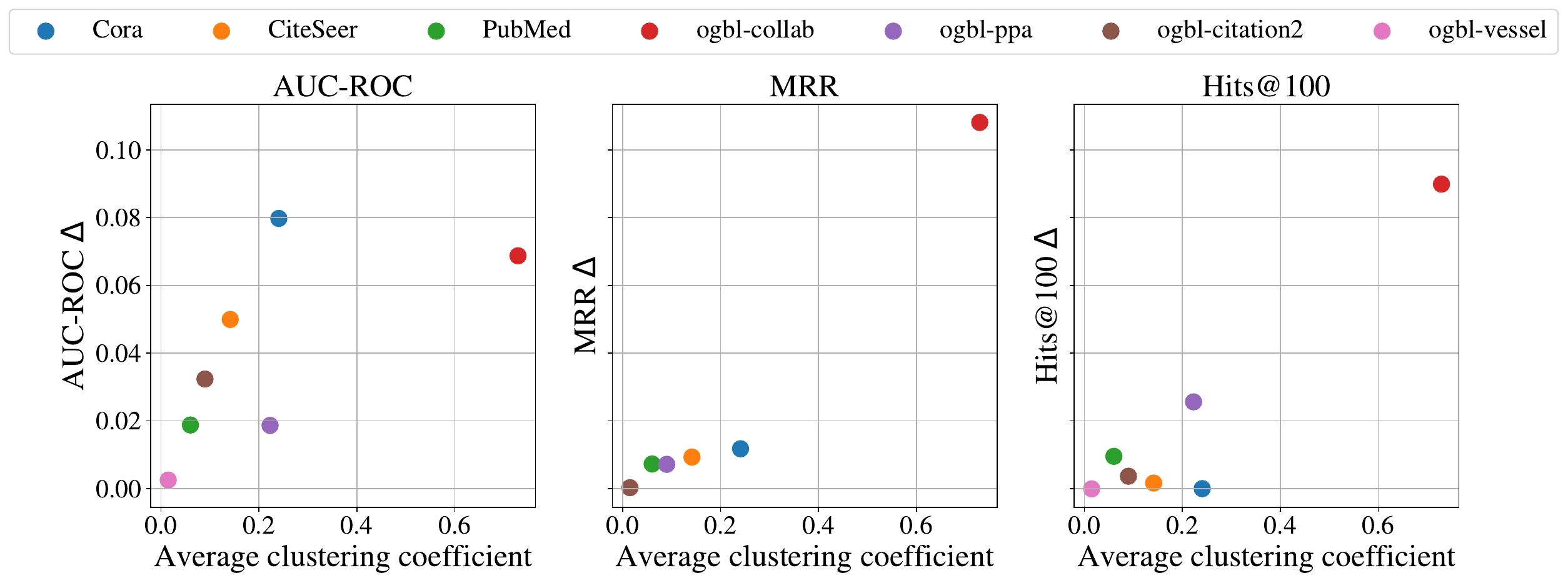}
    \end{subfigure}
    \Description[Performance of removing repulsion completely]{There are two plots. The top is a grouped bar chart with graph dataset names on the x-axis. For each dataset, there are two bars: one for the time decrease when removing repulsion from LINE and another for node2vec. The bottom plot contains three scatter plots, one for each of: AUC-ROC, MRR, and Hits@100. In each scatter plot, a point is a dataset, the x-axis is the average clustering coefficient, and the y-axis is the change in the metric value when repulsion is removed for LINE.}
    \caption{We find that in practice, removing repulsion altogether significantly reduces training time, as shown in the top figure, and can even improve performance. Specifically, removing repulsion performs especially well for real-world graphs that are globally sparse but locally dense (high clustering coefficient). In the bottom figure, we show the change in performance when repulsion is removed relative to vanilla LINE. For almost all graphs, AUC-ROC, MRR, and Hits@$k$ all increase; the increase is most prominent in the case of ogbl-collab, which has high local density.}
    \label{fig:clustering}
\end{figure}

\subsubsection{More generally, repulsion is needed, and dimension regularization provides a scalable solution.} To generalize beyond the seven real-world graphs, we examine the performance of our augmentation on two-block SBM graphs in which we toggle local density. Figure \ref{fig:sbm} shows that when the within-block edge probability is much greater than the between-block probability, all variants perform well. However, as local density decreases and the boundary between the two blocks erodes, repulsion is needed as indicated by the gap between \texttt{I} and \texttt{II}$^0$; the figure shows that dimension regularization (\texttt{II}) is an effective repulsion mechanism. In LINE, dimension regularization outperforms all variants, and in node2vec, dimension regularization interpolates between the vanilla and repulsion-less variants. For LINE, the performance dip for \texttt{II}$^0$ indicates that performance may be unstable when repulsion is removed, likely due to the possibility of embedding collapse.

\begin{figure}[h]
    \centering
    \includegraphics[width=\columnwidth]{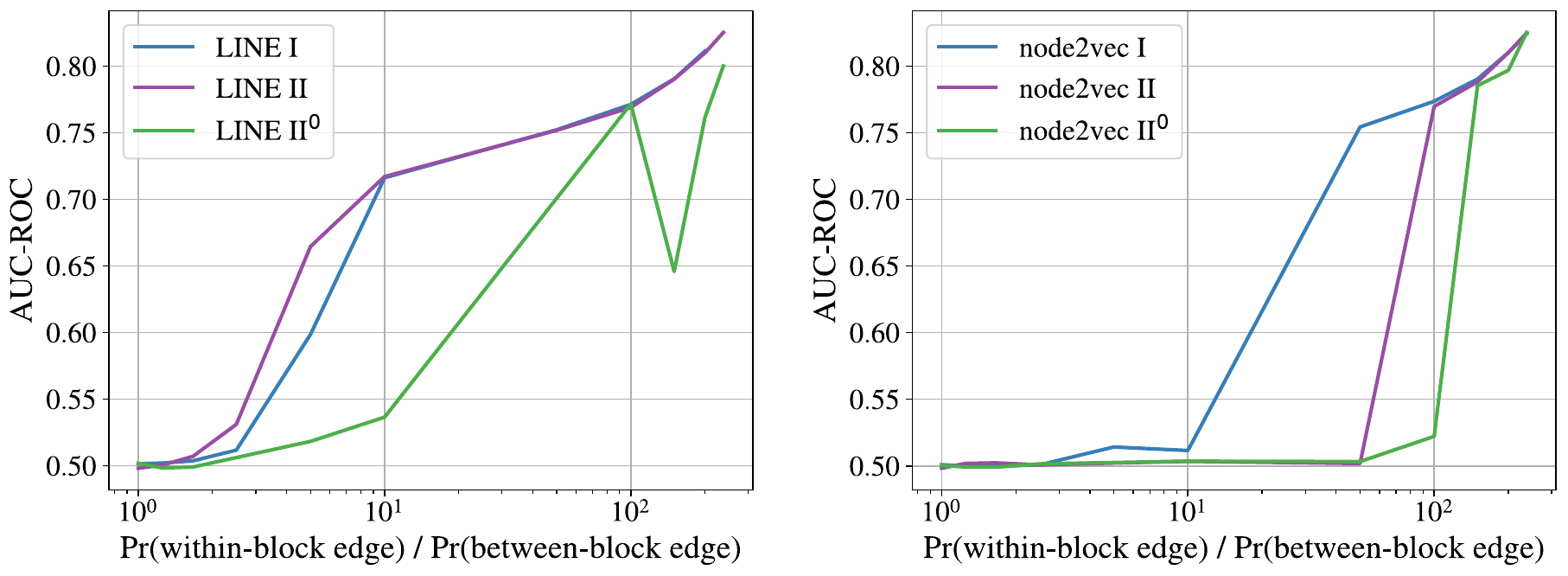}
    \Description[Removing repulsion for SBMs]{The figure consists of two line charts. The first line chart shows the performance of each variant of LINE (AUC-ROC on the y-axis) as a function of the within-block / between-block edge-probability ratio. There is a line for each of the three variants of LINE. The right chart is the same, but for node2vec.}
    \caption{We show that as we decrease local density in a two-block SBM, the need for repulsion increases and dimension regularization provides an effective repulsion mechanism. In both figures, lower values on the x-axis imply fewer edges within blocks (lower local density) and more edges between blocks. When local density is high, the blocks are easily distinguishable and all variants perform well, but as local density decreases, repulsion is needed to separate the blocks and the performance of the repulsion-less special case (\texttt{II}$^0$) erodes. Dimension regularization (\texttt{II}) performs the best in LINE and interpolates between the vanilla and repulsion-less variants in node2vec.}
    \label{fig:sbm}
\end{figure}

\section{Discussion}\label{sec:discussion}
In this section, we discuss the broader applicability of our theoretical result as well as the limitations of our dimension-regularization meta-algorithm.

\paragraph{Broader Applicability of Theoretical Result}
While we focus on the skip-gram loss in this work, our theoretical contribution connecting skip-gram repulsion and dimension-mean regularization, as detailed in Section \ref{sec:sg-norm}, is applicable beyond skip-gram.
First, observe that the skip-gram loss we introduce in Eq.~\eqref{eqn:sg-loss} is a generic cross-entropy logistic pair-classifier loss. Though we focus on the setting where positive and negative pairs are defined via the skip-gram model, dimension-mean regularization can be efficiently used for repulsion in Eq.~\eqref{eqn:sg-loss} regardless of how positive and negative pairs are defined, provided the number of positive pairs is $o(n^2)$, where $n$ is the vocabulary size. 
Second, dimension-mean regularization is also an efficiency bypass for certain formulations of Noise Contrastive Estimation (NCE)~\cite{gutmann2012noise}. Past works, such as the similarity-based graph embedding algorithm VERSE~\cite{tsitsulin2018verse}, have reformulated NCE in terms of negative sampling, using the exact SGNS loss in Eq.~\eqref{eqn:sg-negative}; for these negative-sampling-based formulations of NCE, the dimension-mean regularizer is an efficient bypass of negative sampling, again provided the number of positive samples is $o(n^2)$.
These two applications of the dimension-mean regularizer as an efficient form of repulsion for cross-entropy and NCE fall directly from our theoretical analysis in Section \ref{sec:sg-norm}.

\paragraph{Meta-algorithm Limitations} While the results in Section \ref{sec:eval} demonstrate that our augmentation framework reduces training time and memory usage while preserving link-prediction performance, we note two limitations. First, our dimension-regularization meta-algorithm introduces two hyperparameters $\lambda$ and $n_{\text{negative}}$, which require tuning. Second, while link-prediction performance is preserved overall, there are graphs, such as Cora, for which performance decreased relative to node2vec. That said, these graphs still benefit from significant decreases in training time and memory.
\section{Conclusion}
We provide a new perspective on dissimilarity preservation in graph representation learning and show that dissimilarity preservation can be achieved via dimension regularization. Our main theoretical finding shows that when node repulsion is most needed and embedding dot products are all increasing, the difference between the original skip-gram dissimilarity loss and the dimension-mean regularizer vanishes. Combined with the efficiency of dimension operations over node repulsions, dimension regularization bypasses the need for SGNS. 
We then introduce a generic algorithm augmentation that prioritizes positive updates. When node repulsion is needed and collapse approaches, the augmentation utilizes dimension regularization instead of SGNS.
Our experimental results show that replacing SGNS with dimension regularization in LINE and node2vec preserves the link-prediction performance of the original algorithms while reducing runtime by up to $23.4$\% and memory by up to $33.3\%$ for OGB benchmark datasets. Furthermore, for LINE, removing dissimilarity preservation altogether, a special case of our augmentation framework, reduces training time by $70.9\%$ on average while also increasing link prediction performance--especially for graphs that are globally sparse but locally dense. However, for graphs with low local density, repulsion is needed, and dimension regularization is an efficient and effective approach.

\bibliographystyle{ACM-Reference-Format}
\bibliography{references}

\clearpage
\appendix
\appendix

\section{Proofs} \label{sec:proofs}

\subsection{Proof for Proposition \ref{prop:ase-regularization}}

\begin{proof}
    Recall that the Frobenius norm of a matrix is equivalent to the trace of the corresponding Gram matrix:
    \begin{align}
        N_{ASE}(X, S) &= \| XX^T \|_F^2\\
        &= \text{Tr}\left(XX^T \left(XX^T\right)^T\right)\\
        &= \text{Tr}\left(XX^TXX^T\right)\\ \label{eqn:trace-cycle}
        &= \text{Tr}\left(X^TXX^TX\right)\\ 
        &= \text{Tr}\left(\left(X^TX\right)\left(X^TX\right)^T\right)\\
        &= \| X^TX \|_F^2,
    \end{align}
\end{proof}
Eq.~\eqref{eqn:trace-cycle} follows from the cyclic property of the trace.

\subsection{Proof for Proposition \ref{prop:guaranteed-oversmoothing}}
\begin{proof}
From gradient descent, we know that $X_i^TX_j$ increases toward infinity for all $i, j \in E$; however, to show that $X_i^TX_j$ increases for \textit{all} $i, j$, we show that the cosine similarity for all pairs of embeddings approaches $1$. We characterize the embedding dynamics in two phases: in the first phase (alignment), the embeddings are initialized near the origin and then converge in direction; then, in the second phase (asymptotic), the embeddings asymptotically move away from the origin while maintaining alignment.

\emph{Phase 1: alignment.} The gradient update rule for $P_{SG}$ is:
\begin{equation} \label{eqn:positive-update}
    X_i^{t+1} = X_i^t + \eta \sum_{j \in N(i)} \sigma\left(-(X_i^t)^TX_j^t\right) X_j^t
\end{equation}
Because the embeddings are initialized sufficiently small, the sigmoid function can be approximated linearly via a first-order Taylor expansion:
\begin{equation}
    \sigma(z) \approx \frac{1}{2} + \frac{1}{4}z
\end{equation}
In this case, the update rule becomes:
\begin{align}
    X_i^{t+1} &\approx X_i^t + \frac{\eta}{2} \sum_{j \in N(i)} \left(1 - \frac{1}{2}\left((X_i^t)^TX_j^t\right)\right) X_j^t
\end{align}
The above gradient is equivalent to performing gradient descent on:
\begin{equation} \label{eqn:entry-wise-eigendecomp}
    P'_{SG} = \|\mathbf{1}_{S > 0}  \odot \left(\mathbf{1} - \frac{1}{2}XX^T\right)\|_F^2
\end{equation}
From \citet{gunasekar2017implicit}, gradient descent for matrix completion is implicitly regularized to yield the minimum nuclear norm (lowest rank) stationary point. In the case where $G$ is connected, minimizing the nuclear norm implies that gradient descent in Eq.~\eqref{eqn:entry-wise-eigendecomp} causes $XX^T$ to approach $2\mathbf{1}$, and thus the dot products between all pairs of embeddings increase and $\mathcal{C} \to 2$. 

\emph{Phase 2: Asymptotic.} To complete the proof we show that once $\mathcal{C} > 0$, the constriction monotonically increases with each gradient-descent update to $P_{SG}$. For any pair of embeddings $X_i^t$ and $X_j^t$, the dot product after a single gradient descent epoch is:
\begin{align}\label{eqn:positive-only-update}
    \begin{split}
    \left<X_i^{t+1}, X_j^{t+1}\right> = &\left(X_i^t + \eta \sum_{k \in N(i)} \sigma\left(-(X_i^t)^TX_k^t\right) X_k^t\right)^T\\
    &\left(X_j^t + \eta \sum_{k \in N(j)} \sigma\left(-(X_j^t)^TX_k^t\right) X_k^t\right)
    \end{split}
\end{align}
If $\mathcal{C} > 0$ at $t$, then all of the dot-product terms from distributing the right-hand side of Eq.~\eqref{eqn:positive-only-update} are positive, and we have $\left<X_i^{t+1}, X_j^{t+1}\right> = \left<X_i^{t}, X_j^{t}\right> + \delta$, where $\delta > 0$. Thus, the dot product between any pair of embeddings strictly increases from step $t$ to $t+1$, and therefore the constriction monotonically increases if $\mathcal{C} > 0$ at step $t$.
\end{proof}

\subsubsection{Supplemental Figures}\label{sec:oversmoothing-validation}
\begin{figure*}
    \centering
    \includegraphics[width = 0.9\linewidth]{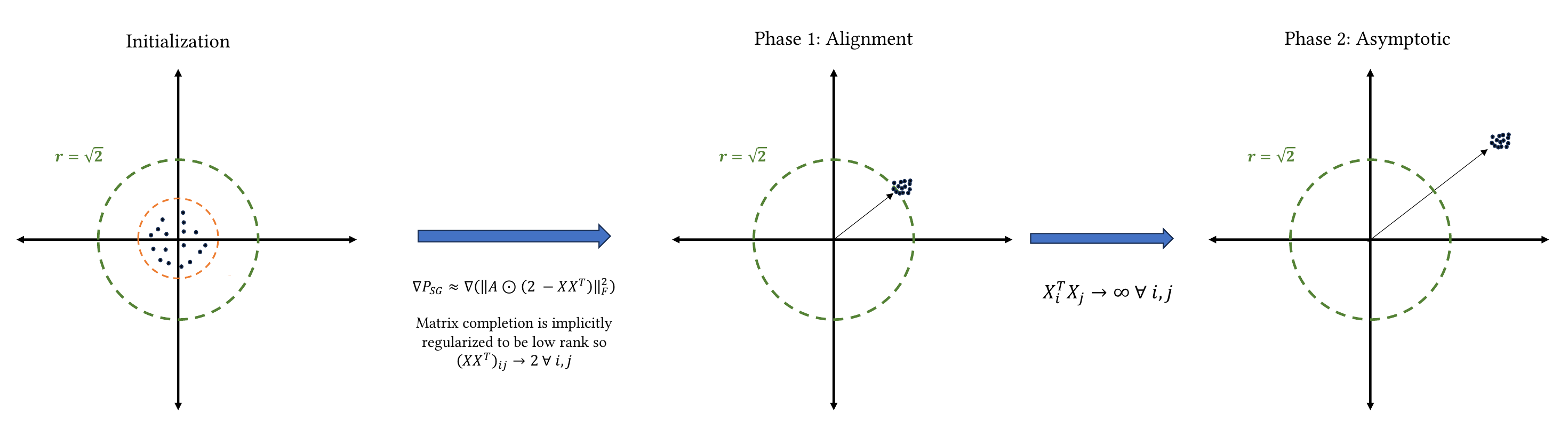}
    \Description[Illustration of proof for Proposition \ref{prop:guaranteed-oversmoothing}]{The figure shows the geometry of graph embeddings as characterized in the proof. The embeddings are shown for three stages: initialization, alignment, and asymptotic. At initialization, the embeddings are shown as a cluster of black points near the origin of an x-y axis. At alignment, the points have clustered at a point on the circumference of a circle centered at the origin with radius equal to the square-root of 2. At the asymptotic stage, the points are clustered at a point even further from the origin.}
    \caption{High-level overview of proof for Proposition \ref{prop:guaranteed-oversmoothing}, which guarantees embedding collapse when only attraction updates are applied. In the beginning, the embeddings are initialized near the origin. Then, in Phase 1, the attraction update rule is approximately gradient descent for matrix completion; given that the latter is implicitly regularized to yield low-rank solutions the embeddings converge in direction. In the second phase, the embeddings asymptotically distance away from the origin in the same direction.}
    \label{fig:oversmoothing-summary}
\end{figure*}

\begin{figure}
    \centering
    \includegraphics[width = 0.8\linewidth]{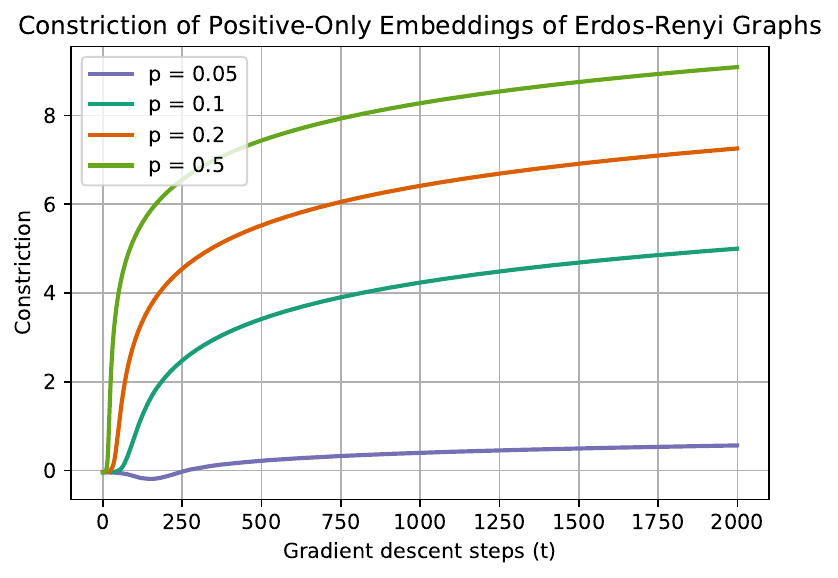}
    \Description[Validation of Proposition \ref{prop:guaranteed-oversmoothing}]{A line plot where the x-axis is the gradient descent step and the y-axis is the Constriction. Each line represents a different Erdos-Renyi graph density. For each density p (0.5, 0.2, 0.1, 0.05), the line starts at the origin and then eventually becomes monotonically increasing with the step count. The curves where p is 0.5, 0.2, or 0.1 increase sharply during the initial gradient descent steps before a more gradual increase. For the smallest p (0.05), the curve initially dips before monotonically increasing above zero.}
    \caption{To empirically validate Proposition \ref{prop:guaranteed-oversmoothing}, we instantiate Erd\"{o}s-R\'{e}nyi networks ($n = 100$) and learn embeddings by only applying the attraction update. The figure shows that for various graph densities, the constriction eventually becomes monotonically increasing.}
    \label{fig:oversmoothing-validation}
\end{figure}
Figure \ref{fig:oversmoothing-summary} provides a high-level summary of the proof for Proposition \ref{prop:guaranteed-oversmoothing}. In the beginning, the embeddings are initialized with norm $\leq b \leq \sqrt{2}$. Then, the embeddings converge in direction given that near the origin, the gradient of $P_{SG}$ is approximately the gradient of a matrix completion problem. Further, gradient descent for matrix completion near the origin is implicitly regularized to yield the lowest rank solution, hence a convergence in the embedding direction. Thereafter, once the dot products between all pairs of nodes are positive in the ``Asymptotic'' phase, $\mathcal{C}$ becomes monotonically increasing.  

We also empirically validate Proposition \ref{prop:guaranteed-oversmoothing} in Figure \ref{fig:oversmoothing-validation}. We initialize Erd\"{o}s-R\'{e}nyi graphs ($n = 100$) of various densities, randomly initialize the embeddings around the origin, and then apply the positive-only update rule in Eq.~\eqref{eqn:positive-update}. Figure \ref{fig:oversmoothing-validation} shows that even when the edge density is $0.05$, the constriction eventually becomes monotonically increasing.

\subsection{Proof for Proposition \ref{prop:all-to-all-approx}}

\begin{proof}
Let us define the matrix of embedding similarities as $K = \sigma\left(XX^T\right)$. Then, the gradient of $N'_{SG}$ is: 
\begin{equation}
    \nabla N'_{SG} =  2KX
\end{equation}
If $\mathbf{1}_{S==0}$ is the indicator matrix where entry $i,j$ is one if $S_{ij} = 0$, then the gradient of $\nabla N_{SG}$ is:
\begin{equation}
    \nabla N_{SG} = \left(\mathbf{1}_{S==0} \odot K\right) X
\end{equation}

The numerator in the proposition can be upper bounded as: 
\begin{align}
    \| \nabla N'_{SG} - \nabla N_{SG}\|^2_F  &= \sum_{i=1}^n \Big\|\sum_{j' \in \{j | S_{ij} > 0\}} \sigma\left(X_i^TX_{j'}\right) X_{j'}\Big\|^2_2\\ 
    &\leq \sum_{i=1}^n \lvert \{j | S_{ij} > 0\} \lvert \beta_{\text{max}}\\
    &\leq m \beta_{\text{max}}
\end{align}
Where in the above, $\beta_{\text{max}}$ is a constant and the upper bound on embedding norm squared, and $m$ is the number of non-zero entries in $S$.

Now we lower bound the denominator. The gradient can be expanded as:
\begin{equation}
    \| \nabla N_{SG}\|^2_F = \sum_{i=1}^n \Big\|\sum_j^n \mathbf{1}_{S_{ij} == 0} K_{ij} X_j\Big\|^2_2
\end{equation}
We can lower bound the norm of the sum by replacing $X_j$ with the projection of $X_j$ onto $X_i$:
\begin{align}
    \| \nabla N_{SG}\|^2_F &= \sum_{i=1}^n \Big\|\sum_j^n \mathbf{1}_{S_{ij} == 0} K_{ij} X_j\Big\|^2_2\\
    &\geq \sum_{i=1}^n \Big\| \sum_{j}^n \mathbf{1}_{S_{ij} == 0} K_{ij} \left(\frac{K_{ij}}{\|X_i\|}\right) \left(\frac{X_i}{\|X_i\|}\right) \Big\|_2^2
\end{align}
Because the dot product between all pairs of embeddings is assumed to be positive, the norm of the sum is at most the sum of the norms:
\begin{align}
    \| \nabla N_{SG}\|^2_F &\geq \sum_{i=1}^n \sum_{j}^n \mathbf{1}_{S_{ij} == 0} \left(\frac{K_{ij}}{\| X_i \|}\right)^4\\
    &\geq \left(\frac{\mathcal{C}^2}{\beta_{\text{max}}}\right)^2\left(n^2 - m\right)
\end{align}
where above $m$ is the number of non-zero entries of $S$. 

Combining the bounds on the numerator and denominator, we have:
\begin{equation}
    \frac{\| \nabla N'_{SG} - \nabla N_{SG}\|^2_F}{\| \nabla N_{SG}\|^2_F} \leq \frac{\beta_{\text{max}}^3}{\mathcal{C}^4} \frac{m}{n^2-m}
\end{equation}
The left term is a constant given the assumption on constriction and non-vanishing or infinite embedding norms. Further, because the graph is sparse ($m$ is $o\left(n^2\right)$), the second term goes to zero as $n \to \infty$.
\end{proof}

\subsection{Proof for Proposition \ref{prop:sg-regularization}}
\subsubsection{Lemma for Proof of Proposition \ref{prop:sg-regularization}}
\begin{lemma} \label{lem:sigmoid-vs-x}
Call $f(x) = \log(1+\exp(x))$. We show that both $f(x) - x \leq \exp(-x)$ and $| \nabla_x (f(x) - x)| \leq \exp(-x)$, i.e. have vanishing exponential tails.
\end{lemma}
\begin{proof}
    First, note that $\log(x) \leq x - 1$. Since $e^{-x}(1+e^x) = 1+e^{-x}$, we have,
    \begin{align*}
        \log(e^{-x}(1+e^x)) &\leq (1+e^{-x}) - 1 \\
        \log(1+\exp(x)) - x &\leq \exp(-x) \\
        f(x) - x &\leq \exp(-x),
    \end{align*}
    where the first line applies the bound on $\log(x)$ to the equality, the second line organizes terms, and the third applies the definition of $f(x)$.
    Now we have $\nabla_x (f(x) - x) = \sigma(x) - 1$. For this,
    \begin{align*}
        \sigma(x) - 1 = \frac{-1}{1+\exp(x)} \geq \frac{-1}{\exp(x)} = -\exp(-x).
    \end{align*}
    Where the inequality follows from reducing the value of the denominator. This concludes the proof. 
\end{proof}

\subsubsection{Proof for Proposition \ref{prop:sg-regularization}}

\begin{proof}
As in the proof for Proposition \ref{prop:all-to-all-approx}, let us define the similarity matrix $K = \sigma\left(XX^T\right)$.

The gradient for $N'_{SG}$, defined in Eq.~\eqref{eqn:sg-all-to-all}, is:
\begin{equation}
\nabla_X N'_{SG} = 2K X
\end{equation}

The constriction $\mathcal{C}$ is the minimum value of the matrix K. From lemma \ref{lem:sigmoid-vs-x}, we know that as constriction increases, the difference between $1$ and each of the entries of $K$ vanishes exponentially. Thus, there is a vanishing difference between $\frac{1}{2}\nabla_X N'_{SG}$ and $\mathbf{1}X$, where $\mathbf{1}$ is the $n \times n$ all-ones matrix. Note that $\mathbf{1}X$ is also the gradient of the dimension regularizer $R$ introduced in Eq.~\eqref{eqn:mean-regularizer}. Putting these together:
\begin{align}
    \left\| \frac{1}{2} \nabla_X N'_{SG} - \nabla R \right\|_2^2 &=  \left\| \left(\sigma\left(XX^T\right) - \mathbf{1} \right) X\right\|_2^2\\
    &\leq \sum_{i=1}^n \left\| \sum_{j=1}^n \left(\sigma\left(X_i^TX_j\right) - 1 \right) X_j \right\|^2\\
    &\leq \sum_{i=1}^n \left(\sum_{j=1}^n \left(\sigma\left(X_i^TX_j\right) - 1 \right)^2 \| X_j\|^2 \right)\\
    & \leq \left(\frac{n}{e^{\mathcal{C}}}\right)^2 \beta_{\text{max}}
\end{align}
In the above, $\beta_{\text{max}}$ is a constant and corresponds to the maximum embedding norm among all embeddings in $X$.

Thus, as $\mathcal{C}$ increases, the difference between the gradient of $\frac{1}{2}N'_{SG}$ and $R$ vanishes exponentially. By extension, the difference between gradient descent on $N'_{SG}$ with a step-size of $\eta$ and gradient descent on $R$ with a step-size of $2\eta$ vanishes with increasing $\mathcal{C}$. 
\end{proof}

\section{Experimental Methodology}\label{sec:appendix-methodology}

\paragraph{Hyperparameter Optimization.} To reduce the search space, we sequentially optimize the learning rate, node2vec random-walk, and augmentation parameters following the routine below. The optimal values are documented in our code repository.
\begin{enumerate}
    \item \textbf{Learning rate.} For each dataset and vanilla algorithm (LINE and node2vec), select the learning rate that maximizes the validation set AUC-ROC among  $\eta \in \{0.001$, $0.01$, $0.1$, $1.0\}$.
    \item \textbf{node2vec random-walk parameters.} For each dataset, run a node2vec grid search over $p, q \in \{0.25, 0.50, 1.0, 2.0, 4.0\}$. For Cora, CiteSeer, and PubMed, select the values $p, q$ that maximize the validation AUC-ROC. For the OGB graphs, select the values that maximize the metric corresponding to the OGB leaderboard (see Table \ref{tab:ogb-metrics}).
    \item \textbf{Augmentation parameters.} For each graph and vanilla model, perform a grid search over $n_{\text{negative}} \in \{5, 10, 100\}$ and $\lambda \in \{0.1, 1.0, 10.0, 100.0\}$. Select the values that maximize the same metrics optimized in the previous step. For Cora, CiteSeer, and PubMed, select the top-five highest-scoring parameter values; for ogbl-collab, select the top four; for the remaining OGB graphs, select the top three. In the final test-set evaluation, the results for $\texttt{II}$ correspond to the best-performing values of $n_{\text{negative}}$ and $\lambda$, as measured on the test set using the metrics specified in step (2).
\end{enumerate}

\section{Supplemental Evaluation}\label{sec:supplemental-eval}

\paragraph{OGB metrics.} In addition to the AUC-ROC metrics shown in the main body, in Table \ref{tab:ogb-metrics}, for the OGB graphs, we also include the metrics reported on the OGB leaderboards. 
\begin{table}[h]
    \centering
    \caption{Performance for OGB graphs according to leaderboard metrics.}
    \begin{tabular}{lllccc}
        \hline
        \multirow{2}{*}{\textbf{Graph}} & \multirow{2}{*}{\textbf{Metric}} & \multirow{2}{*}{\textbf{Model}} & \multicolumn{3}{c}{\textbf{Performance}} \\
        \cline{4-6}
        & & & \texttt{I} & \texttt{II}$^0$ & \texttt{II} \\
        \hline
        \multirow{2}{*}{\textbf{ogbl-collab}} & \multirow{2}{*}{Hits@50} & node2vec & 0.25 & 0.14 & 0.17 \\
        & & LINE & 0.05 & 0.14 & 0.19 \\
        \hline
        \multirow{2}{*}{\textbf{ogbl-ppa}} & \multirow{2}{*}{Hits@100} & node2vec & 0.15 & 0.02 & 0.02 \\
        & & LINE & 0.01 & 0.04 & 0.08 \\
        \hline
        \multirow{2}{*}{\textbf{ogbl-citation2}} & \multirow{2}{*}{MRR} & node2vec & 0.32 & 0.07 & 0.16 \\
        & & LINE & 0.03 & 0.03 & 0.08 \\
        \hline
        \multirow{2}{*}{\textbf{ogbl-vessel}} & \multirow{2}{*}{AUC-ROC} & node2vec & 0.96 & 0.83 & 0.87 \\
        & & LINE & 0.50 & 0.50 & 0.50 \\
        \hline
    \end{tabular}
    \label{tab:ogb-metrics}
\end{table}

\paragraph{Non-uniform negative sampling.} 

\begin{table}[]
    \caption{When SGNS and our augmentation are instantiated with non-uniform negatives, our augmentation still preserves link-prediction performance as measured by AUC-ROC. Column \texttt{I} is SGNS instantiated with the common configuration of $\alpha=3/4$ and $k=5$. Column $\texttt{II}$ is our augmentation instantiated with $\Vec{p}$ as the probability vector of $P_{3/4}$ (see Section \ref{sec:reweighting}). Entries marked as ``-'' denote out-of-memory errors on OGB datasets when $k=5$, highlighting the GPU memory cost of SGNS.}
    \label{table:weighted-negatives}
    \centering
    \begin{tabular}{l|cc|cc}
        \hline
        \multirow{2}{*}{\textbf{Graph}} &
        \multicolumn{2}{c|}{\textbf{node2vec AUC-ROC}} &
        \multicolumn{2}{c}{\textbf{LINE AUC-ROC}} \\
        & \texttt{I} & \texttt{II} & \texttt{I} & \texttt{II} \\
        \hline
        \textbf{CiteSeer} & 0.68 & 0.79 & 0.61 & 0.62 \\
        \textbf{Cora} & 0.84 & 0.85 & 0.60 & 0.66 \\
        \textbf{PubMed} & 0.76 & 0.84 & 0.69 & 0.72 \\
        \textbf{ogbl-collab} & - & 0.96 & 0.71 & 0.78 \\
        \textbf{ogbl-ppa} & - & 0.98 & 0.92 & 0.76 \\
        \textbf{ogbl-citation2} & - & 0.94 & 0.67 & 0.58 \\
        \textbf{ogbl-vessel} & - & 0.88 & 0.50 & 0.50 \\
        \hline
    \end{tabular}
\end{table}

In Section~\ref{sec:reweighting}, we note that SGNS introduces the two hyperparameters $\alpha$ and $k$. $\alpha$ enables non-uniform sampling over nodes, where the probability of sampling a node is proportional to the node degree$^\alpha$. $k$ specifies the number of negative samples per positive pair. In the main body, we analyze both SGNS and our augmentation with the uniform distribution ($\alpha=0$, $k=1$, $\Vec{p}=\Vec{1}$) as this is a common instantiation of SGNS. However, in practice, it is also common to set $\alpha=3/4$ and $k=5$, based on promising empirical results from~\citet{mikolov2013distributed}. In Table~\ref{table:weighted-negatives}, we show that when we instantiate SGNS and our augmentation with non-uniform sampling, our augmentation preserves link-prediction performance. When we set $k=5$ for SGNS, we obtain out-of-memory (OOM) errors on the OGB graphs, illustrating the memory cost of node contrast.

\balance
\end{document}